\documentclass[journal]{IEEEtran}

\usepackage{amsmath,amssymb,amsfonts}
\usepackage{algorithmic}
\usepackage{graphicx}
\usepackage{algorithm,algorithmic}
\usepackage{hyperref}
\usepackage{amsthm}

\newtheorem{theorem}{Theorem}%{Theorem}%[section]
\newtheorem{lemma}{Lemma}
\newtheorem{corollary}{Corollary}
\newtheorem{assumption}{Assumption}
\newtheorem{remark}{Remark}

\begin{document}

\title{Communication Efficient Federated Learning with Linear Convergence on Heterogeneous Data}
\author{Jie Liu and Yongqiang Wang, ~\IEEEmembership{Senior Member,~IEEE} \thanks{The work was supported in part by the National Science Foundation under Grants CCF-2106293, CCF-2215088, CNS-2219487, CCF-2334449, and CNS-2422312 (Corresponding author: Yongqiang Wang).} \thanks{The authors are with the Department of Electrical and Computer Engineering, Clemson University, Clemson, SC 29634, USA (e-mail: jie9@clemson.edu; yongqiw@clemson.edu).} }

\markboth{ }%
{Shell \MakeLowercase{\textit{et al.}}: Bare Demo of IEEEtran.cls for IEEE Journals}

\maketitle

\begin{abstract}
By letting local clients perform multiple local updates before communicating with a parameter server, modern federated learning algorithms such as FedAvg tackle the communication bottleneck problem in distributed learning and have found many successful applications. However, this asynchrony between local updates and communication also leads to a ``client-drift'' problem when the data is heterogeneous (not independent and identically distributed), resulting in errors in the final learning result. In this paper, we propose a federated learning algorithm, which is called FedCET, to ensure accurate convergence even under heterogeneous distributions of data across clients. Inspired by the distributed optimization algorithm NIDS, we use learning rates to weight information received from local clients to eliminate the ``client-drift''. We prove that under appropriate learning rates, FedCET can ensure linear convergence to the exact solution. Different from existing algorithms which have to share both gradients and a drift-correction term to ensure accurate convergence under heterogeneous data distributions, FedCET only shares one variable, which significantly reduces communication overhead.  Numerical comparison with existing counterpart algorithms confirms the effectiveness of FedCET.
\end{abstract}

\begin{IEEEkeywords}
Federated Learning, Heterogeneous Data Distributions, Linear Convergence
\end{IEEEkeywords}

\IEEEpeerreviewmaketitle

\section{Introduction}

Federated learning has been a hot research topic due to its various advantages over centralized learning in data security, privacy preservation, and communication efficiency \cite{Li_Challenge,FL_Survey2,Xuanyu_Cao1}.  Since first proposed in \cite{pmlr-v54-mcmahan17a}, federated learning has been widely applied in many fields, ranging from neural networks \cite{neural_network1,neural_network2,neural_network3}, wireless networks \cite{wireless_network1,wireless_network2,wireless_network3}, healthcare \cite{healthcare1,healthcare2,healthcare3}, to the Internet of things \cite{Internet_of_Things,Internet_of_Things2,Internet_of_Things3}. Unlike centralized machine learning where all data are aggregated to a data center, federated learning allows training data to stay on individual clients \cite{pmlr-v54-mcmahan17a}. Each client performs multiple local training steps based on its local dataset before sharing information with a parameter server, which then aggregates these information and distributes the results to clients to ensure that all clients can learn the global model. In recent decades, various aspects of federated learning have been intensively studied, including communication efficiency \cite{FedPAQ,A_Mitra_Linear,pmlr-v130-haddadpour21a}, learning rate design \cite{Locally_Mukherjee,Kim_Adpative,Pan_Wang_Li_Wang_Tang_Zhao_2023}, and privacy/security protection \cite{data_security,data_security2,data_security3}.

In the canonical federated learning architecture (see FedAvg in \cite{pmlr-v54-mcmahan17a}), to reduce the communication overhead, individual clients perform multiple local updates based on local data before communicating with the parameter server. However, this leads to ``client-drift'' when the data distribution is heterogeneous across the clients, i.e., the final results will deviate from the optimal solution (see \cite{pmlr-v119-karimireddy20a,pmlr-v119-malinovskiy20a,pmlr-v130-charles21a,NEURIPS2020_4ebd440d}). More specifically, when the data distributions are heterogeneous across the clients (not independent and identically distributed), the expected loss functions of individual clients are different from the global loss function, which makes the optimal solution of the global loss function differ from that of each client's local loss function. After multiple local training steps, local model parameters move toward the optimal solutions of their local loss functions, leading to a drift from the optimal solution of the global loss function. As a result, popular federated learning algorithms, such as FedAvg, cannot converge to the exact optimal solution of the global loss function. It is worth noting that diminishing learning rates can be used to mitigate this optimization error. However, using decreasing learning rates will result in slow convergence.

Methods to obviate the ``client-drift'' in the heterogeneous data distributed case have been reported in recent years.  In \cite{pmlr-v119-karimireddy20a}, SCAFFOLD is proposed to ensure exact convergence in the heterogeneous-distribution case by transmitting an additional control variable to correct the ``client-drift'' in local updates.  In \cite{A_Mitra_Aggregated}, FedTrack is proposed which shares both local model parameters and local gradient information between the parameter server and the clients to eliminate the ``client-drift'' by incrementally aggregating the gradients of each client's loss function. In \cite{A_Mitra_Linear}, FedLin with gradient sparsification is proposed to reduce the communication overhead in FedTrack. However, it transmits both model parameters and (compressed) gradient information to guarantee exact convergence.

All existing algorithms in \cite{A_Mitra_Linear,pmlr-v119-karimireddy20a,A_Mitra_Aggregated} have to share one additional variable besides the gradient variable between clients and the parameter server, which doubles the communication overhead compared with the classic federated learning algorithm FedAvg, where only the gradient variable has to be shared.  In this paper, inspired by a well-known distributed optimization algorithm NIDS \cite{Ming_Yan1}, we propose an algorithm that can tackle the ``client-drift'' problem by only sharing one variable between local clients and the parameter server. Specifically, we use the learning rate to weight information received from individual clients, which is proven to be effective to address the ``client-drift'' problem. It is worth noting that \cite{Ming_Yan1} only allows one local update in each communication round, and hence, we have to significantly revise the algorithm and proof techniques therein to allow multiple local updates in each round which is crucial for federated learning. 

Under the $\mu$-strongly-convex and $L$-smooth assumptions for each client's local loss function, we prove that when the learning rate satisfies certain conditions, the proposed algorithm can guarantee linear convergence to the exact optimal solution even when the data distribution is heterogeneous across the clients. Compared with existing results for federated learning under heterogeneous data in \cite{A_Mitra_Linear,pmlr-v119-karimireddy20a,A_Mitra_Aggregated}, which have to share both the gradient and an additional control variable between the parameter server and local clients in each communication round, our algorithm only needs to share one variable in each communication round, and hence, has a much higher communication efficiency. Numerical evaluations show that even with the reduced communication, our algorithm can still achieve faster convergence than existing counterpart algorithms.

\section{Preliminaries}
\subsection{Notations}

$\mathbf{0}_{n\times p}\in\mathbb{R}^{n\times p}$ denotes the matrices (vectors for $p=1$) with all elements being $0$. $\mathbf{1}_N\in\mathbb{R}^N$ denotes the $N$-dimensional vector with all element being $1$. $\mathbf{I}_{N}\in\mathbb{R}^{N\times N}$ is the identical matrix. For the convenience of expression, we use $\mathbf{I}$  to represent $\mathbf{I}_{N}$ when the dimension is clear from the context. For two matrices $A,B\in\mathbb{R}^{n\times p}$, we define their inner product as $\langle A,B\rangle={\bf tr}(A^{\bf T}B)$ and the Frobenius norm as $\Vert A\Vert=\sqrt{\langle A,A\rangle}$. In addition, for a symmetric matrix $Q\in\mathbb{R}^{n\times n}$, we define $\langle A,B\rangle_{Q}={\bf tr}(A^{\bf T}QB)$ and $\Vert A\Vert_{Q}=\sqrt{\langle A,A\rangle_{Q}}$. For two symmetric matrices $A,B\in\mathbb{R}^{n\times n}$, we use $A\succ B$ ($A\succeq B$) to denote $A-B$ is positive definite (or positive semidefinite). $A^{\dagger}$ denotes the pseudo inverse of $A$. Given two positive integers $a$ and $b$, $a\ \textbf{mod}\ b$ represents the remainder of the division of $a$ by $b$.

\subsection{Problem Settings}

We consider the following federated learning problem over a client set $\mathcal{S}=\{1,2,\cdots,N\}$ as follows:
\begin{align}\label{FL_problem_cdc}
\min_{x\in\mathbb{R}^n}f(x)=\frac{1}{N}\sum^{N}_{i=1}f_i(x),
\end{align}
where $f_i:\mathbb{R}^n\rightarrow\mathbb{R}$ is the local loss function of client $i\in\mathcal{S}$. The loss function $f_i(x)$ is solely dependent on the local training data of client $i\in\mathcal{S}$ and is only accessible by client $i$. We assume that the loss function $f_i(x)$ of client $i\in\mathcal{S}$ satisfies the following assumptions:

\begin{assumption}\label{smooth_assumption}
The loss function $f_i(x)$ of the client $i\in\mathcal{S}$ is $L$-smooth over $\mathbb{R}^n$, that is, there exists a constant $L>0$ such that
\begin{align*}
\Vert \nabla f_i(x)-\nabla f_i(y)\Vert&\leq L \Vert x-y\Vert,
\end{align*}
holds for any $x,y\in\mathbb{R}^n$.
\end{assumption}

\begin{assumption}\label{strong_convex_assumption}
The loss function $f_i(x)$ of the client $i\in\mathcal{S}$ is $\mu$-strongly-convex over $\mathbb{R}^n$, that is, there exists a constant $\mu>0$ such that
\begin{align*}
\langle\nabla f_i(x)-\nabla f_i(y),x-y\rangle&\geq \mu\Vert x-y\Vert^2,
\end{align*}
holds for any $x,y\in\mathbb{R}^n$.
\end{assumption}

From Assumption \ref{smooth_assumption} and Assumption \ref{strong_convex_assumption}, we know that the optimal solution $x^*=\arg\min_{x\in\mathbb{R}^n}f(x)$ of the global loss function $f(x)$ exists and is unique. Moreover, the optimal solution $x^*$ satisfies 
\begin{align*}
\nabla f({x}^*)=\mathbf{0}_n.
\end{align*}
Each client performs local updates using its local dataset and shares information with the parameter server to ensure that all clients can converge to the global optimal solution $x^*$.

\section{Main Results}
In this section, we propose the federated learning algorithm to solve (\ref{FL_problem_cdc}) and establish its linear convergence to the exact optimal solution. The proposed federated learning algorithm, which we name FedCET, has the advantage of a lower communication overhead over the federated learning algorithms in existing works \cite{A_Mitra_Linear,pmlr-v119-karimireddy20a,A_Mitra_Aggregated}, which will be explained later.

\subsection{Algorithm Description}
Some parameters and initial values should be explained before introducing FedCET (see Algorithm \ref{new_algorithm_new_condition}). The number of local training steps, the weight parameter, and the learning rate are denoted as $\tau$, $c$, and $\alpha$, respectively. The local training step $\tau$ is a positive integer ($\tau\in\mathbb{Z}^{+})$. The learning rate $\alpha$ is obtained by running Algorithm \ref{step_size_search}. 
\begin{algorithm}%[H]
   \caption{Learning Rate Search}
   \label{step_size_search}
\begin{algorithmic}
\STATE{\textbf{Initialization}: the search stepsize $h$, the initial learning rate of $\alpha$ as $\alpha_0<\min\{\frac{1}{2\tau L},\frac{\mu^2}{2\tau(1+\frac{2}{\tau})^{2\tau-2}L^3}$, $\frac{\mu}{5\tau(1+\frac{2}{\tau})^{2\tau-2} L^2}\}$.}
   \STATE{\textbf{while} {$\Big(1-\tau\mu\alpha+\tau  L^2(\tau\alpha-\frac{2}{ \mu })(1+\frac{2}{\tau})^{2\tau-2}\alpha>0$ \& $(1-\tau L\alpha)\tau\mu\alpha+\tau^3 L^4(\tau\alpha-\frac{2}{ \mu })(1+\frac{2}{\tau})^{2\tau-2}\alpha^3>0\Big)$} }
  \begin{align*}
\alpha=\alpha+h.
\end{align*}
   \textbf{end while}
\STATE{\textbf{Output} the learning rate $\alpha=\alpha-h$.}
\end{algorithmic}
\end{algorithm}
In FedCET, the model parameter of client $i\in\mathcal{S}$ is denoted as $x_i(t)$. We assume that at the iteration $t$, the client $i$ updates its model parameter $x_i(t+1)$ based on its loss function $f_i(x)$, its previous model parameter, and received information from the parameter server (if any). For convenience of expression, we assume that, at time $t=-1$, each client $i\in\mathcal{S}$ shares information 
\begin{align*}
2x_i(-1)-x_i(-2)-\alpha\nabla f_i(x_i(-1))+\alpha \nabla f_i(x_i(-2))
\end{align*}
with the parameter server, where
\begin{align*}
 x_i(-1)=x_i(-2)-\alpha\nabla f_i(x_i(-2))
\end{align*}
and $x_i(-2)\in\mathbb{R}^{n}$ are the initial values of client $i\in\mathcal{S}$. Then, each client $i\in\mathcal{S}$ updates $x_i(0)$ as follows
\begin{align*}
       x_{i}&(0)=\frac{c\alpha}{N}\sum^{N}_{j=1}\Big\{ 2x_j(-1)-x_j(-2)-\alpha \nabla f_j(x_j(-1))\nonumber\\
       &+\alpha \nabla f_j(x_j(-2))\Big\}+(1-c\alpha)\Big\{2x_i(-1)-x_i(-2)\nonumber\\
       &-\alpha \nabla f_i(x_i(-1))+\alpha \nabla f_i(x_i(-2))\Big\}.
   \end{align*}
After setting initial values $x_i(-1)$ and $x_i(0)$ ($i\in\mathcal{S}$), we present the detailed steps of FedCET in Algorithm \ref{new_algorithm_new_condition}.

\begin{algorithm}%[H]
   \caption{FedCET }
   \label{new_algorithm_new_condition}
\begin{algorithmic}
\STATE{\textbf{Initialization}:  local training period $\tau\in\mathbb{Z}^{+}$, the learning rate $\alpha$ obtained from Algorithm \ref{step_size_search}, and the weight parameter $0<c\leq\frac{\mu}{2\mu\alpha+{8}}$.}
   \FOR{$t=0$ {\bfseries to} $T$}
   \FOR{each client $i=1,2,\cdots,N$ in parallel}
   
   \IF{$t+1\ \textbf{mod}\ \tau =0$}\STATE The parameter server receives 
   \begin{align*}
    2x_j(t)&-x_j(t-1)-\alpha \nabla f_j(x_j(t))\\
    &+\alpha \nabla f_j(x_j(t-1))
   \end{align*}
   from client $j\in\mathcal{S}$ and transmits 
   \begin{align*}
    \frac{1}{N}\sum^{N}_{j=1}\Big\{&2x_j(t)-x_j(t-1)-\alpha \nabla f_j(x_j(t))\\
    &+\alpha \nabla f_j(x_j(t-1))\Big\}
   \end{align*}
   to all clients.
   Each client $i$ updates its local state as
   \begin{align}\label{equation_state_new_condition}
       x_{i}&(t+1)=\frac{c\alpha}{N}\sum^{N}_{j=1}\Big\{ 2x_j(t)-x_j(t-1)\nonumber\\
       &-\alpha \nabla f_j(x_j(t))+\alpha \nabla f_j(x_j(t-1))\Big\}\nonumber\\
       &+(1-c\alpha)\Big\{2x_i(t)-x_i(t-1)\nonumber\\
       &-\alpha \nabla f_i(x_i(t))+\alpha \nabla f_i(x_i(t-1))\Big\}.
   \end{align}
   \ELSE{ \STATE Each client $i$ does local training}
   \begin{align}\label{equation_state_new_condition2}
       x_{i}(t+1)= &2x_i(t)-x_i(t-1)-\alpha \nabla f_i(x_i(t))\nonumber\\
       &+\alpha \nabla f_i(x_i(t-1)).
   \end{align}
   \ENDIF
   \ENDFOR
   \ENDFOR
\end{algorithmic}
\end{algorithm}

In FedCET, only one $n$-dimension vector is transmitted periodically between the parameter server and clients after every $\tau$ local updates. Specifically, at each communication round ($t+1=k\tau$, $k\geq 0$), the vector $$2x_i(t)-x_i(t-1)-\alpha\nabla f_i(x_i(t))+\alpha \nabla f_i(x_i(t-1))$$ is transmitted from the client $i\in\mathcal{S}$ to the parameter server, which then transmits a vector $$\frac{1}{N}\sum^{N}_{i=1}\{ 2x_i(t)-x_i(t-1)-\alpha \nabla f_i(x_i(t))+\alpha \nabla f_i(x_i(t-1))\}$$ back to all clients. Then, each client performs $\tau$ times local training steps based on (\ref{equation_state_new_condition}) and (\ref{equation_state_new_condition2}). Each client uses the weight parameter $c$ and the learning rate $\alpha$ to adjust its weight for local training in (\ref{equation_state_new_condition}).  The detailed explanation for designing $\alpha$ and $c$ can be found in the proof of Corollary \ref{corollary1} and Theorem \ref{most_important}. It is key to eliminate the ``client-drift'' and guarantee convergence to the exact optimal solution. Next, we provide the detailed convergence analysis of FedCET.

\subsection{Convergence Analysis}
The matrix form of (\ref{equation_state_new_condition}) and (\ref{equation_state_new_condition2}) is used to analyze the convergence of FedCET. More specifically, by defining
\begin{align}
x(t)=&[x_1(t),x_2(t),\cdots,x_N(t)]^{\bf T},\label{matrix_form_x_x}\\
\nabla f(x)=&[\nabla f_1(x_1),f_2(x_2),\cdots,\nabla f_N(x_N)]^{\bf T},\label{matrix_form_f_f}\\
d(t)=&\frac{1}{\alpha}(x(t-1)-x(t))-\nabla f(t-1),\label{definition_of_d(t)}
\end{align}
where $x=[x_1,x_2,\cdots,x_N]^{\bf T}\in\mathbb{R}^{N\times n}$ and $\nabla f(t)$ is used to represent $\nabla f(x(t))$ for the convenience of expression, we can write the updates of (\ref{equation_state_new_condition}) and (\ref{equation_state_new_condition2}) in a more compact matrix form in (\ref{matrix_algorithm_equivalent_form}), which is summarized in Lemma \ref{matrix_form_x_d} below:
\begin{lemma}\label{matrix_form_x_d}
The processes (\ref{equation_state_new_condition}) and (\ref{equation_state_new_condition2}) in
Algorithm \ref{new_algorithm_new_condition} can be expressed as the following matrix form
\begin{equation}\label{matrix_algorithm_equivalent_form}
\left\{
\begin{aligned}
d(t+1)=&d(t)+c(\mathbf{I}-W(t+1))\{x(t)-\alpha\nabla f(t)\\
&-\alpha d(t)\},\\
x(t+1)=&x(t)-\alpha\nabla f(x(t))-\alpha d(t+1),
\end{aligned}
\right.
\end{equation}
where \begin{equation}\label{matrix_W}
W(t+1)=\left\{
\begin{aligned}
\frac{1}{N}\mathbf{1}_{N}\mathbf{1}^{\bf T}_{N}, &\quad t+1=\tau k,\\
 \mathbf{I},\ \quad  &\quad t+1\ne\tau k.
\end{aligned}
\right.
\end{equation}
\end{lemma}
\begin{proof}
From (\ref{equation_state_new_condition}), (\ref{equation_state_new_condition2}), (\ref{matrix_form_x_x}), (\ref{matrix_form_f_f}), and (\ref{matrix_W}), we know that
\begin{align}
x(t+1)=\widetilde{W}(t+1)\Big\{&2x(t)-x(t-1)-\alpha\nabla f(x(t))\nonumber\\
&+\alpha\nabla f(x(t-1))\Big\},\label{expression_x(t+1)}
\end{align}
holds, where $$\widetilde{W}(t+1)=(1-c\alpha)\mathbf{I}+c\alpha W(t+1).$$ Then, based on the definition of $d(t)$ in (\ref{definition_of_d(t)}), we have
\begin{align}
&d(t+1)\nonumber\\
=&\frac{1}{\alpha}(x(t)-x(t+1))-\nabla f(t)\nonumber\\
=&\frac{1}{\alpha}x(t)-\nabla f(t)-\frac{1}{\alpha}\widetilde{W}(t+1)\{2x(t)-x(t-1)\nonumber\\
&-\alpha\nabla f(t)+\alpha\nabla f(t-1)\}\nonumber\\
=&d(t)+c(\mathbf{I}-W(t+1))\{2x(t)-x(t-1)-\alpha\nabla f(t)\nonumber\\
&+\alpha\nabla f(t-1)\}\nonumber\\
=&d(t)+c(\mathbf{I}-W(t+1))\{x(t)-\alpha\nabla f(t)-\alpha d(t)\}.\label{expression_d(t+1)}
\end{align}
Using (\ref{expression_x(t+1)}) and (\ref{expression_d(t+1)}), we can derive that
\begin{align*}
x(t+1)=&\widetilde{W}(t+1)\{x(t)-\alpha\nabla f(t)-\alpha d(t)\}\\
=&-c\alpha(\mathbf{I}-W(t+1))\{x(t)-\alpha\nabla f(t)-\alpha d(t)\}\\
&+x(t)-\alpha\nabla f(t)-\alpha d(t)\\
=&x(t)-\alpha\nabla f(t)-\alpha d(t+1),
\end{align*}
which completes the proof.
\end{proof}

Next, we prove that the fixed point of (\ref{matrix_algorithm_equivalent_form}) and the optimal value $x^*=\arg\min_{x\in\mathbb{R}^n}f(x)$ of the federated learning problem (\ref{FL_problem_cdc}) are the same. To this end, we first provide Lemma \ref{fix_point_lemma1} below:
\begin{lemma}\label{fix_point_lemma1}
$(\mathbf{d}^*,\mathbf{x}^*)$ is the fixed point of (\ref{matrix_algorithm_equivalent_form}) if and only if the following two equalities are satisfied:
\begin{equation}\label{fix_point}
\left\{
\begin{aligned}
\mathbf{d}^*+\nabla f(\mathbf{x}^*)&=\mathbf{0}_{N\times n},\\
(\mathbf{I}-\frac{1}{N}\mathbf{1}_{N}\mathbf{1}^{\bf T}_{N})\mathbf{x}^*&=\mathbf{0}_{N\times n},
\end{aligned}
\right.
\end{equation}
where $\mathbf{x}^*\in\mathbb{R}^{N\times n}$ and $\nabla f(\mathbf{x}^*)$ is defined based on (\ref{matrix_form_f_f}).
\end{lemma}

\begin{proof}
We prove Lemma \ref{fix_point_lemma1} from two aspects:
\begin{itemize}
\item[$\Rightarrow$] If $(\mathbf{d}^*,\mathbf{x}^*)$ is the fixed point of (\ref{iteration_matrix1}) and (\ref{iteration_matrix2}), for any $t+1\geq 1$, we have
\begin{align*}
\mathbf{d}^*=&\mathbf{d}^*+c(\mathbf{I}-W(t+1))\{\mathbf{x}^*-\alpha\nabla f(\mathbf{x}^*)-\alpha\mathbf{d}^*\},\\
\mathbf{x}^*=&\mathbf{x}^*-\alpha\nabla f(\mathbf{x}^*)-\alpha\mathbf{d}^*.
\end{align*}
From the definition of $\widetilde{W}(t+1)$ in (\ref{matrix_W}), one can easily obtain (\ref{fix_point}).
\item[$\Leftarrow$] If (\ref{fix_point}) is satisfied, then under $(d(t),x(t))=(\mathbf{d}^*,\mathbf{x}^*)$, for $t+1\neq k\tau$, we have  
\begin{align*}
d(t+1)&=\mathbf{d}^*+c(\mathbf{I}-\mathbf{I})\{\mathbf{x}^*-\alpha\nabla f(\mathbf{x}^*)-\alpha \mathbf{d}^*\}\\
&=\mathbf{d}^*
\end{align*}
and, for $t+1=k\tau$, we have
\begin{align*}
d(t+1)=&\mathbf{d}^*+c(\mathbf{I}-\frac{1}{N}\mathbf{1}_{N}\mathbf{1}^{\bf T}_{N})\{\mathbf{x}^*-\alpha\nabla f(\mathbf{x}^*)\\
&-\alpha \mathbf{d}^*\}\\
=&\mathbf{d}^*.
\end{align*}
Moreover, $x(t+1)=\mathbf{x}^*-\alpha\nabla f(\mathbf{x}^*)-\alpha\mathbf{d}^*=\mathbf{x}^*$. Thus, $(\mathbf{d}^*,\mathbf{x}^*)$ is the fixed point of (\ref{matrix_algorithm_equivalent_form}).
\end{itemize}
\end{proof}

Then, we need the following Lemma \ref{fix_point_lemma2} from \cite{Ming_Yan1} to characterize the property of the optimal value $x^*=\arg\min_{x\in\mathbb{R}^n}f(x)$ of the federated learning problem (\ref{FL_problem_cdc})
\begin{lemma}\label{fix_point_lemma2}
The fix point $\mathbf{x}^*$ in Lemma \ref{fix_point_lemma1} is the optimal solution to (\ref{FL_problem_cdc}), i.e., $\mathbf{x}^*=[x^*,x^*,\cdots,x^*]^{\bf T}\in\mathbb{R}^{N\times n}$, if and only if there exists $\mathbf{p}^*$ such that
\begin{equation}\label{fix_point_2}
\left\{\begin{aligned}
(\mathbf{I}-\frac{1}{N}\mathbf{1}_{N}\mathbf{1}^{\bf T}_{N})\mathbf{p}^*+\nabla f(\mathbf{x}^*)&=\mathbf{0}_{N\times n},\\
(\mathbf{I}-\frac{1}{N}\mathbf{1}_{N}\mathbf{1}^{\bf T}_{N})\mathbf{x}^*&=\mathbf{0}_{N\times n}.
\end{aligned}
\right.
\end{equation}
\end{lemma}

From Lemma \ref{fix_point_lemma1} and Lemma \ref{fix_point_lemma2}, we know that $(\mathbf{d}^*,\mathbf{x}^*)$ with $\mathbf{d}^*=(\mathbf{I}-\frac{1}{N}\mathbf{1}_{N}\mathbf{1}^{\bf T}_{N})\mathbf{p}^*$ and $\mathbf{x}^*=[x^*,x^*,\cdots,x^*]^{\bf T}$ is the fixed point of  (\ref{matrix_algorithm_equivalent_form}). Since the information is transmitted with a period $\tau$ in FedCET, the relationship between $(d(k\tau+\tau),x(k\tau+\tau))$ and $(d(k\tau),x(k\tau))$ should be derived to analyze its convergence. We can derive the following Lemma \ref{matrix_iteration_lemma} from Lemma \ref{matrix_form_x_d}.

\begin{lemma}\label{matrix_iteration_lemma}
For any communication round $k\geq 0$, we can have the following relationship
\begin{align}
d(\tau k+\tau)=&d(\tau k)+c(\mathbf{I}-\frac{1}{N}\mathbf{1}_{N}\mathbf{1}^{\bf T}_{N})\{x(\tau k)-\tau\alpha {\nabla} f(\tau k)\nonumber\\
&-\tau\alpha d(\tau k)+A(k)\},\label{iteration_matrix1}\\
x(\tau k+\tau)=&x(\tau k)-\tau\alpha {\nabla} f(\tau k)-\tau\alpha d(\tau k+\tau)\nonumber\\
&+A(k)+B(k),\label{iteration_matrix2}
\end{align}
where 
\begin{align*}
A(k)=&(\tau-1)\alpha {\nabla} f(\tau k)-\alpha\sum^{\tau-1}_{h=1} {\nabla} f(\tau k+h),\\
B(k)=&(\tau-1)\alpha d(\tau k+\tau)-(\tau-1)\alpha d(\tau k).
\end{align*}
\end{lemma}
\begin{proof}
Based on the definition of ${W}(t+1)$ in (\ref{matrix_W}), we know that ${W}(h)=\mathbf{I}$ holds for any $\tau k+1\leq h<\tau k+\tau$. Thus, for any $k\geq 0$, we have
\begin{align*}
d(\tau k+1)=&d(\tau k),\\
x(\tau k+1)=&x(\tau k)-\alpha {\nabla} f(\tau k)-\alpha d(\tau k),
\end{align*}
from (\ref{matrix_algorithm_equivalent_form}) in Lemma \ref{matrix_form_x_d}. 

For any $1<s\leq\tau-2$, we assume
\begin{align*}
d(\tau k+s)=&d(\tau k),\\
x(\tau k+s)=&x(\tau k)-\alpha\sum^{s-1}_{h=0} {\nabla} f(\tau k+h)-s\alpha d(\tau k).
\end{align*}
Then, from (\ref{matrix_algorithm_equivalent_form}) in Lemma \ref{matrix_form_x_d}, we know that
\begin{align*}
d(\tau k+s+1)=&d(\tau k),\\
x(\tau k+s+1)=&x(\tau k)-\alpha\sum^{s}_{h=0} {\nabla} f(\tau k+h)\\
&-(s+1)\alpha d(\tau k).
\end{align*}
Thus, using mathematical induction, we can obtain
\begin{align*}
d(\tau k+\tau-1)=&d(\tau k),\\
x(\tau k+\tau-1)=&x(\tau k)-\alpha\sum^{\tau-2}_{h=0} {\nabla} f(\tau k+h)\\
&-(\tau-1)\alpha d(\tau k).
\end{align*}
Because ${W}(k\tau+\tau)=\frac{1}{N}\mathbf{1}_{N}\mathbf{1}^{\bf T}_{N}$ always holds, we have
\begin{align*}
d(\tau k+\tau)=&d(\tau k)+c(\mathbf{I}-\frac{1}{N}\mathbf{1}_{N}\mathbf{1}^{\bf T}_{N})\Big\{x(\tau k+\tau-1)\\
&-\alpha d(\tau k)-\alpha {\nabla} f(\tau k+\tau-1)\Big\},\\
x(\tau k+\tau)=&x(\tau k+\tau-1)-\alpha {\nabla} f(\tau k+\tau-1)\\
&-\alpha d(\tau k+\tau).
\end{align*}
Based on the definition of $A(k)$ and $B(k)$, we have
\begin{align*}
d(\tau k+\tau)=&d(\tau k)+c(\mathbf{I}-\frac{1}{N}\mathbf{1}_{N}\mathbf{1}^{\bf T}_{N})\{x(\tau k)-\tau\alpha {\nabla} f(\tau k)\\
&-\tau\alpha d(\tau k)+A(k)\},\\
x(\tau k+\tau)=&x(\tau k)-\tau\alpha {\nabla} f(\tau k)-\tau\alpha d(\tau k+\tau)\\
&+A(k)+B(k),
\end{align*}
which completes the proof.
\end{proof}

From Lemma \ref{fix_point_lemma1}, Lemma \ref{fix_point_lemma2}, and Lemma \ref{matrix_iteration_lemma}, $\Vert x(\tau k)-\mathbf{x}^*\Vert$ can be utilized to analyze the convergence of FedCET. Thus, we present the mathematical properties of 
\begin{align*}
\Vert x(\tau k+\tau)-\mathbf{x}^*\Vert^2\quad \text{and} \quad \Vert d(k\tau)-\mathbf{d}^*\Vert^2
\end{align*}
in the following Theorem \ref{most_important}.

\begin{theorem}\label{most_important}
For any communication round $k\geq 0$, if $0<c\leq\frac{\mu}{2(\mu\alpha+{4})}$ and $0<\alpha\leq 2\tau L$ hold, we have
\begin{align}
&(1-\tau\mu\alpha){\Vert} x(\tau k+\tau)-\mathbf{x}^*{\Vert}^2_{(\tau\alpha\mathbf{I})^{-1}}\nonumber\\
&+{\Vert} d(k\tau+\tau)-\mathbf{d}^*{\Vert}^2_{M+(1-\tau\mu\alpha)\tau\alpha\mathbf{I}}\nonumber\\
\leq&{\Big\{}1+ L\mu\tau^2\alpha^2+\frac{2\tau^3}{\mu}(1+\frac{2}{\tau})^{2\tau-2} L^4\alpha^3-2\tau\mu\alpha\nonumber\\
&-\tau^4(1+\frac{2}{\tau})^{2\tau-2} L^4\alpha^4 {\Big\}}{\Vert} x(\tau k)-\mathbf{x}^*{\Vert}^2_{(\tau\alpha\mathbf{I})^{-1}}\nonumber\\
&+{\Vert} d(\tau k)-\mathbf{d}^*{\Vert}^2_{M+(\frac{2}{\tau\mu\alpha}-1){\tau^3(1+\frac{2}{\tau})^{2\tau-2}} L^2\alpha^3\mathbf{I}},\label{theorem_equation1}
\end{align}
where $M=c^{-1}(\mathbf{I}-\frac{1}{N}\mathbf{1}_{N}\mathbf{1}^{\bf T}_{N})^{\dagger}-\alpha \mathbf{I}$.
\end{theorem}

\begin{proof}
 See in Appendix \ref{most_important_proof}.
\end{proof}

Next, we prove that the learning rate $\alpha$ designed by Algorithm \ref{step_size_search} can ensure a linear convergence of FedCET.

\begin{corollary}\label{corollary1}
Under the learning rate obtained from Algorithm \ref{step_size_search}, there exist $0<\rho_1<1$ and $0<\rho_2<1$ such that
\begin{align*}
&(1-\tau\mu\alpha){\Vert} x(\tau k+\tau)-\mathbf{x}^*{\Vert}^2_{(\tau\alpha\mathbf{I})^{-1}}+{\Vert} d(k\tau+\tau)-\mathbf{d}^*{\Vert}^2_{M_1}\\
\leq&\rho_1(1-\tau\mu\alpha){\Vert} x(\tau k)-\mathbf{x}^*{\Vert}^2_{(\tau\alpha\mathbf{I})^{-1}}+\rho_2{\Vert} d(\tau k)-\mathbf{d}^*{\Vert}^2_{M_1}.
\end{align*}
holds for the iterates in Algorithm \ref{new_algorithm_new_condition}, where $M_1=M+(1-\tau\mu\alpha)\tau\alpha\mathbf{I}$ and $M=c^{-1}(\mathbf{I}-\frac{1}{N}\mathbf{1}_{N}\mathbf{1}^{\bf T}_{N})^{\dagger}-\alpha \mathbf{I}$. In addition, we have
\begin{align*}
{\Vert} x(\tau k)-\mathbf{x}^*{\Vert}^2\leq&\rho^k{\Vert} x(0)-\mathbf{x}^*{\Vert}^2+\frac{\rho^k\tau\alpha}{1-\tau\mu\alpha}{\Vert} d(0)-\mathbf{d}^*{\Vert}^2_{M_1},
\end{align*}
where $\rho=\max\{\rho_1,\rho_2\}$.
\end{corollary}

\begin{proof} 
See in Appendix \ref{proof_corollary1}.
\end{proof}

\begin{remark}
From the derivation of Theorem \ref{most_important}, we know that the learning rate has to satisfy
\begin{equation}\label{convergence_inequality}
\left\{
\begin{aligned}
1-\tau\mu\alpha>& 1+ L\mu\tau^2\alpha^2+\frac{2\tau^3}{\mu}(1+\frac{2}{\tau})^{2\tau-2} L^4\alpha^3\\
&-2\tau\mu\alpha-\tau^4(1+\frac{2}{\tau})^{2\tau-2} L^4\alpha^4\\
1-\tau\mu\alpha>&(\frac{2}{\tau\mu\alpha}-1){\tau^2(1+\frac{2}{\tau})^{2\tau-2}} L^2\alpha^2,
\end{aligned}
\right.
\end{equation}
to guarantee the linear convergence of FedCET. The initial learning rate 
\begin{align*}
\alpha_0<\min\{\frac{1}{2\tau L},\frac{\mu^2}{2\tau(1+\frac{2}{\tau})^{2\tau-2}L^3},\frac{\mu}{5\tau(1+\frac{2}{\tau})^{2\tau-2} L^2}\}
\end{align*}
in Algorithm \ref{step_size_search} satisfies (\ref{convergence_inequality}) (see more details in the proof of Corollary \ref{corollary1}). The aim of Algorithm \ref{step_size_search} is to find the learning rate satisfying (\ref{convergence_inequality}) that is as large as possible to achieve fast convergence.  With a shorter search stepsize $h$ and hence a finer search granularity, a larger learning rate can be found by Algorithm \ref{step_size_search}. However, a shorter search stepsize $h$ also results in a higher computational burden.
\end{remark}

\begin{remark}
In FedCET, only one $n$-dimensional vector $$2x_i(t)-x_i(t-1)-\alpha\nabla f_i(x_i(t))+\alpha \nabla f_i(x_i(t-1))$$ is transmitted from the client $i\in\mathcal{S}$ to the parameter server, which then transmits only one $n$-dimensional vector $$\frac{1}{N}\sum^{N}_{i=1}\{ 2x_i(t)-x_i(t-1)-\alpha \nabla f_i(x_i(t))+\alpha \nabla f_i(x_i(t-1))\}$$ back to all clients. This is different from existing federated learning algorithms for heterogeneous data distributions such as SCAFFOLD \cite{pmlr-v119-karimireddy20a}, FedTrack \cite{A_Mitra_Aggregated}, and FedLin \cite{A_Mitra_Linear}, where two $n$-dimension variables have to be transmitted between the parameter sever and local clients in each communication round. 
\end{remark}

\section{Numerical Evaluations}

In this section, we use FedCET to solve an empirical risk minimization problem with the global loss function given by
\begin{align}\label{simulation_euqation}
\min_{x\in\mathbb{R}^n}f(x)=\frac{1}{N}\sum^{N}_{i=1}\{\frac{1}{n_i}\sum^{n_i}_{j=1}\Vert M_i x-b_{ij}\Vert^2+r_i\Vert x\Vert^2\}.
\end{align}
The above optimization problem occurs in many applications. For example, it can be used to formulate the distributed estimation problem \cite{wang_poor}: each client $i\in\mathcal{S}$ has $n_i$ noisy measurements of the parameter $b_{ij} = M_ix + w_{ij}$ for $j=\{1, 2,...,n_i\}$. $M_i\in\mathbb{R}^{n\times n}$ is the measurement matrix of the client $i\in\mathcal{S}$, $r_i$ is a nonnegative regularization parameter of client $i\in\mathcal{S}$, and $w_{ij}\in\mathbb{R}^{n}$ is the measurement noise associated with the measurement $b_{ij}\in\mathbb{R}^{n}$. The estimation of the parameter $x\in\mathbb{R}^{n}$ can be solved using the empirical risk minimization problem formulated in (\ref{simulation_euqation}). To facilitate calculation, we consider $M_i=\mathbf{I}_n$ and $r_i=1$ for any $i\in\mathcal{S}$. Then, the loss function becomes $f_i(x)=\frac{1}{n_i}\sum^{n_i}_{j=1}\Vert  x-b_{ij}\Vert^2+\Vert x\Vert^2$, which satisfies Assumption \ref{smooth_assumption} and Assumption \ref{strong_convex_assumption}.

We consider $10$ clients ($\mathcal{S}=\{1,2,\cdots,10\}$) with each client having $10$ noisy measurements ($n_i=10$). We set the dimension of the parameter $x$ as $60$ ($n=60$) and the number of local training step as $2$ ($\tau=2$). Each dimension of $b_{ij}$ is randomly generated from $[-10,10]$. We use FedCET, FedTrack \cite{A_Mitra_Aggregated}, and SCAFFOLD \cite{pmlr-v119-karimireddy20a} to solve this empirical risk minimization problem. The convergence error 
\begin{align*}
e(k)=\Vert \frac{1}{N}\sum^{N}_{i=1}x_i(k\tau)-x^*\Vert,
\end{align*}
where $k$ is the $k^{th}$ communication round, is used to measure the convergence performance. The learning rate of FedCET is obtained from Algorithm \ref{step_size_search} with the search stepsize $h=0.001\alpha_0$. The learning rate of FedTrack is selected as $\alpha=\frac{1}{18\tau L}$. In addition, the global and local learning rates of SCAFFOLD are selected as $\alpha_g=1$ and $\alpha_l=\frac{1}{81\tau L}$, respectively. The learning rates of FedTrack and SCAFFOLD are selected according to the learning-rate rules prescribed in \cite{A_Mitra_Aggregated} and \cite{pmlr-v119-karimireddy20a}. Precise gradients (full-batch) are utilized in all algorithms. The results are summarized in Fig \ref{fig1}. Clearly, the convergence of our proposed algorithm is faster than that of FedTrack and SCAFFOLD, while our algorithm only shares a half of messages that need to be shared in SCAFFOLD and FedTrack.

\begin{figure}[H]
\centerline{\includegraphics[width=\columnwidth]{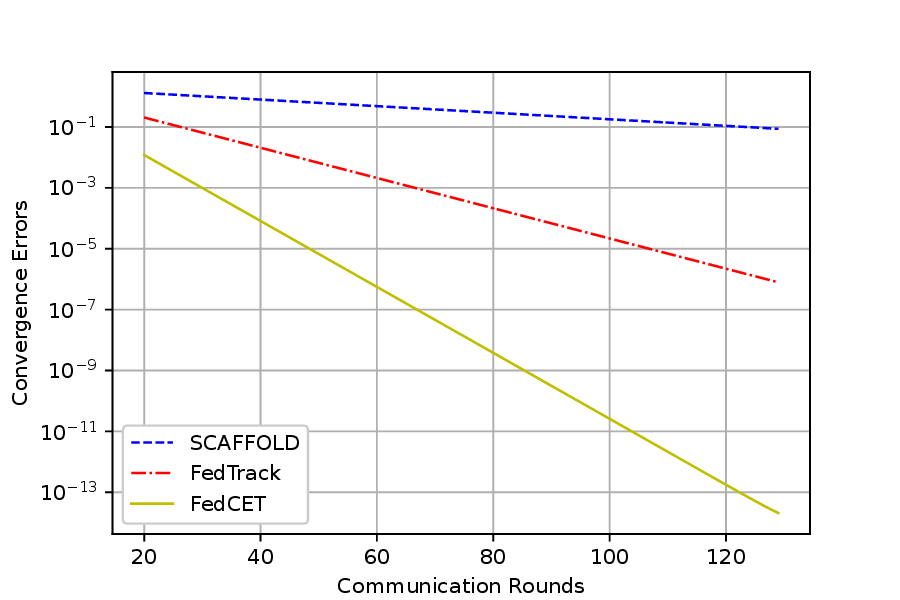}}
\caption{Comparisons of FedCET with FedTrack and SCAFFOLD.}
\label{fig1}
\end{figure}

\section{Conclusions}
In this paper, FedCET is proposed to solve federated learning problems under heterogeneous distributions. Sufficient conditions are provided for the learning rate to guarantee the linear convergence of FedCET to the exact optimal solution. Compared with existing counterpart algorithms for federated learning for heterogeneous data distributions which share two variables between the parameter server and clients in each communication round, our proposed algorithm only shares one variable in each communication round. Numerical results show that even with the reduced communication overhead, our algorithm can still achieve a faster convergence than existing counterpart algorithms.

\appendices

\section{Proof of Theorem \ref{most_important}}\label{most_important_proof}
From (\ref{iteration_matrix1}) and (\ref{iteration_matrix2}), we know that
\begin{align}
&\langle d(k\tau+\tau)-\mathbf{d}^*,x(\tau k+\tau)-\mathbf{x}^*\rangle\nonumber\\
= &\langle d(k\tau+\tau)-\mathbf{d}^*,x(\tau k)-\tau\alpha {\nabla} f(\tau k)\nonumber\\
&-\tau\alpha d(\tau k+\tau)+A(k)+B(k)-\mathbf{x}^*\rangle\nonumber\\
= &\langle d(k\tau+\tau)-\mathbf{d}^*,B(k)-\tau\alpha d(\tau k+\tau)+\tau\alpha d(\tau k)\rangle\nonumber\\
 &+\langle d(k\tau+\tau)-\mathbf{d}^*,x(\tau k)-\mathbf{x}^*\nonumber\\
 &-\tau\alpha {\nabla} f(\tau k)-\tau\alpha d(\tau k)+A(k)\rangle\nonumber\\
= &\langle d(k\tau+\tau)-\mathbf{d}^*,c^{-1}(\mathbf{I}-\frac{1}{N}\mathbf{1}_{N}\mathbf{1}^{\bf T}_{N})^{\dagger}\{d(\tau k+\tau)-d(\tau k)\}\rangle\nonumber\\
 &+\langle d(k\tau+\tau)-\mathbf{d}^*,\alpha d(\tau k)-\alpha d(\tau k+\tau)\rangle\nonumber\\
 =&\langle d(k\tau+\tau)-\mathbf{d}^*,d(\tau k+\tau)-d(\tau k)\rangle_M,\label{appendix_1_e1}
\end{align}
where $M=c^{-1}(\mathbf{I}-\frac{1}{N}\mathbf{1}_{N}\mathbf{1}^{\bf T}_{N})^{\dagger}-\alpha \mathbf{I}$. Then, from (\ref{iteration_matrix1}), (\ref{iteration_matrix2}), and (\ref{appendix_1_e1}), we have
\begin{align}
&\langle x(\tau k)-\mathbf{x}^*,{\nabla} f(\tau k)-{\nabla} f(\mathbf{x}^*) \rangle\nonumber\\
=&\langle x(\tau k)-\mathbf{x}^*,(\tau \alpha)^{-1}\{x(\tau k)-x(\tau k+\tau)\nonumber\\
&-\tau\alpha d(\tau k+\tau)+A(k)+B(k)\}-{\nabla} f(\mathbf{x}^*) \rangle\nonumber\\
=&\langle x(\tau k)-\mathbf{x}^*,(\tau \alpha)^{-1}\{x(\tau k)-x(\tau k+\tau)\nonumber\\
&+A(k)+B(k)\} \rangle+\langle x(\tau k+\tau)-\mathbf{x}^*,\mathbf{d}^*- d(\tau k+\tau) \rangle\nonumber\\
&+\langle x(\tau k)-x(\tau k+\tau),\mathbf{d}^*- d(\tau k+\tau) \rangle\nonumber\\
=&\langle x(\tau k)-\mathbf{x}^*,(\tau \alpha)^{-1}\{A(k)+B(k)\} \rangle\nonumber\\
&+\langle x(\tau k)-\mathbf{x}^*,(\tau \alpha)^{-1}\{x(\tau k)-x(\tau k+\tau)\} \rangle\nonumber\\
&+\langle x(\tau k)-x(\tau k+\tau),\mathbf{d}^*- d(\tau k+\tau) \rangle\nonumber\\
&-\langle d(k\tau+\tau)-\mathbf{d}^*,d(\tau k+\tau)-d(\tau k)\rangle_M\nonumber\\
=&\langle x(\tau k)-x(\tau k+\tau),(\tau \alpha)^{-1}\{x(\tau k)-\mathbf{x}^*\}\nonumber\\
&+\mathbf{d}^*- d(\tau k+\tau) \rangle+\langle x(\tau k)-\mathbf{x}^*,A(k)+B(k) \rangle_{(\tau\alpha\mathbf{I})^{-1}}\nonumber\\
&-\langle d(k\tau+\tau)-\mathbf{d}^*,d(\tau k+\tau)-d(\tau k)\rangle_M\nonumber\\
=&\langle x(\tau k)-x(\tau k+\tau),(\tau \alpha)^{-1}\{x(\tau k+\tau)+\tau\alpha {\nabla} f(\tau k)\nonumber\\
&+\tau\alpha d(\tau k+\tau)-A(k)-B(k)-\mathbf{x}^*\}+\mathbf{d}^*- d(\tau k+\tau) \rangle\nonumber\\
&+\langle x(\tau k)-\mathbf{x}^*,A(k)+B(k) \rangle_{(\tau\alpha\mathbf{I})^{-1}}\nonumber\\
&-\langle d(k\tau+\tau)-\mathbf{d}^*,d(\tau k+\tau)-d(\tau k)\rangle_M\nonumber\\
=&\langle x(\tau k)-x(\tau k+\tau),(\tau \alpha)^{-1}\{x(\tau k+\tau)+\tau\alpha {\nabla} f(\tau k)\nonumber\\
&-\mathbf{x}^*\}+\mathbf{d}^*\rangle-\langle x(\tau k)-x(\tau k+\tau),A(k)+B(k)\rangle_{(\tau\alpha\mathbf{I})^{-1}}\nonumber\\
&+\langle x(\tau k)-\mathbf{x}^*,A(k)+B(k) \rangle_{(\tau\alpha\mathbf{I})^{-1}}\nonumber\\
&-\langle d(k\tau+\tau)-\mathbf{d}^*,d(\tau k+\tau)-d(\tau k)\rangle_M\nonumber\\
=&\langle x(\tau k)-x(\tau k+\tau),x(\tau k+\tau)-\mathbf{x}^*\rangle_{(\tau\alpha\mathbf{I})^{-1}}\nonumber\\
&+\langle x(\tau k)-x(\tau k+\tau),{\nabla} f(\tau k)-{\nabla} f(\mathbf{x}^*)\rangle\nonumber\\
&+\langle x(\tau k+\tau)-\mathbf{x}^*,A(k)+B(k) \rangle_{(\tau\alpha\mathbf{I})^{-1}}\nonumber\\
&-\langle d(k\tau+\tau)-\mathbf{d}^*,d(\tau k+\tau)-d(\tau k)\rangle_M. \label{appendix_1_e2}
\end{align}
For  any symmetric matrix $Q\in\mathbb{R}^{N\times N}$, based on the definition of $\langle\cdot,\cdot\rangle_{Q}$, we have 
\begin{align}
\Vert a+b\Vert^2_{Q}=\Vert a\Vert^2_{Q}+\Vert b\Vert^2_{Q}+2\langle a,b\rangle_{Q},\label{square_inner_product}
\end{align}
where $a,b\in\mathbb{R}^{N\times n}$. Through using (\ref{square_inner_product}) to analyze the terms $\langle x(\tau k)-x(\tau k+\tau),x(\tau k+\tau)-\mathbf{x}^*\rangle_{(\tau\alpha\mathbf{I})^{-1}}$ and $\langle d(k\tau+\tau)-\mathbf{d}^*,d(\tau k+\tau)-d(\tau k)\rangle_M$ in (\ref{appendix_1_e2}), we can derive the following equation
\begin{align}
&2\langle x(\tau k)-\mathbf{x}^*,{\nabla} f(\tau k)-{\nabla} f(\mathbf{x}^*) \rangle\nonumber\\
&-2\langle x(\tau k)-x(\tau k+\tau),{\nabla} f(\tau k)-{\nabla} f(\mathbf{x}^*)\rangle\nonumber\\
=&2\langle x(\tau k)-x(\tau k+\tau),x(\tau k+\tau)-\mathbf{x}^*\rangle_{(\tau\alpha\mathbf{I})^{-1}}\nonumber\\
&+2\langle x(\tau k+\tau)-\mathbf{x}^*,A(k)+B(k) \rangle_{(\tau\alpha\mathbf{I})^{-1}}\nonumber\\
&+2\langle d(k\tau+\tau)-\mathbf{d}^*,d(\tau k)-d(\tau k+\tau)\rangle_M\nonumber\\
=&\Vert x(\tau k)-\mathbf{x}^*\Vert^2_{(\tau\alpha\mathbf{I})^{-1}}-\Vert x(\tau k)-x(\tau k+\tau)\Vert^2_{(\tau\alpha\mathbf{I})^{-1}}\nonumber\\
&-\Vert x(\tau k+\tau)-\mathbf{x}^*\Vert^2_{(\tau\alpha\mathbf{I})^{-1}}+\Vert d(\tau k)-\mathbf{d}^*\Vert^2_{M}\nonumber\\
&-\Vert d(k\tau+\tau)-\mathbf{d}^*\Vert^2_{M}-\Vert d(\tau k)-d(\tau k+\tau)\Vert^2_{M}\nonumber\\
&+2\langle x(\tau k+\tau)-\mathbf{x}^*,A(k)+B(k) \rangle_{(\tau\alpha\mathbf{I})^{-1}}.\label{appendix_1_e3}
\end{align}
Rearranging terms in (\ref{appendix_1_e3}), we can derive that
\begin{align}
&\Vert x(\tau k+\tau)-\mathbf{x}^*\Vert^2_{(\tau\alpha\mathbf{I})^{-1}}+\Vert d(k\tau+\tau)-\mathbf{d}^*\Vert^2_{M}\nonumber\\
\leq&\Vert x(\tau k)-\mathbf{x}^*\Vert^2_{(\tau\alpha\mathbf{I})^{-1}}-\Vert x(\tau k)-x(\tau k+\tau)\Vert^2_{(\tau\alpha\mathbf{I})^{-1}}\nonumber\\
&+\Vert d(\tau k)-\mathbf{d}^*\Vert^2_{M}-\Vert d(\tau k)-d(\tau k+\tau)\Vert^2_{M}\nonumber\\
&+2\langle x(\tau k)-x(\tau k+\tau),{\nabla} f(\tau k)-{\nabla} f(\mathbf{x}^*)\rangle\nonumber\\
&-2\langle x(\tau k)-\mathbf{x}^*,{\nabla} f(\tau k)-{\nabla} f(\mathbf{x}^*) \rangle\nonumber\\
&+2\langle x(\tau k+\tau)-\mathbf{x}^*,A(k)+B(k) \rangle_{(\tau\alpha\mathbf{I})^{-1}}.\label{second_appendix1}
\end{align}
Then, we need to analyze the properties of  $2\langle x(\tau k)-x(\tau k+\tau),{\nabla} f(\tau k)-{\nabla} f(\mathbf{x}^*)\rangle-2\langle x(\tau k)-\mathbf{x}^*,{\nabla} f(\tau k)-{\nabla} f(\mathbf{x}^*) \rangle$ of (\ref{second_appendix1}). From (\ref{fix_point_lemma1}), (\ref{iteration_matrix1}), and (\ref{iteration_matrix2}), we have
\begin{align}
&2\langle x(\tau k)-x(\tau k+\tau),{\nabla} f(\tau k)-{\nabla} f(\mathbf{x}^*)\rangle\nonumber\\
&-2\langle x(\tau k)-\mathbf{x}^*,{\nabla} f(\tau k)-{\nabla} f(\mathbf{x}^*) \rangle\nonumber\\
=&2\langle x(\tau k)-x(\tau k+\tau),(\tau \alpha)\{{\nabla} f(\tau k)-{\nabla}   f(\mathbf{x}^*)\}\rangle_{(\tau\alpha\mathbf{I})^{-1}}\nonumber\\
&-2\langle x(\tau k)-\mathbf{x}^*,{\nabla} f(\tau k)-{\nabla} f(\mathbf{x}^*) \rangle\nonumber\\
=&-2\langle x(\tau k)-x(\tau k+\tau),(\tau \alpha)\{{\nabla} f(\mathbf{x}^*)-{\nabla} f(\tau k)\}\rangle_{(\tau\alpha\mathbf{I})^{-1}}\nonumber\\
&-2\langle x(\tau k)-\mathbf{x}^*,{\nabla} f(\tau k)-{\nabla} f(\mathbf{x}^*) \rangle\nonumber\\
=&{\Vert}  x(\tau k)-x(\tau k+\tau){\Vert}^2_{(\tau\alpha\mathbf{I})^{-1}}-\Vert x(\tau k)-x(\tau k+\tau)\nonumber\\
&+(\tau \alpha)\{{\nabla} f(\mathbf{x}^*)-{\nabla} f(\tau k)\}{\Vert}^2_{(\tau\alpha\mathbf{I})^{-1}}\nonumber\\
&+{\Vert}  (\tau \alpha)\{{\nabla} f(\mathbf{x}^*)-{\nabla} f(\tau k)\}{\Vert}^2_{(\tau\alpha\mathbf{I})^{-1}}\nonumber\\
&-2\langle x(\tau k)-\mathbf{x}^*,{\nabla} f(\tau k)-{\nabla} f(\mathbf{x}^*) \rangle\nonumber\\
=&-\Vert \tau\alpha d(\tau k+\tau)-\tau \alpha \mathbf{d}^*-A(k)-B(k){\Vert}^2_{(\tau\alpha\mathbf{I})^{-1}}\nonumber\\
&+{\Vert}  x(\tau k)-x(\tau k+\tau){\Vert}^2_{(\tau\alpha\mathbf{I})^{-1}}+{\Vert} {\nabla} f(\mathbf{x}^*)-{\nabla} f(\tau k){\Vert}^2_{\tau\alpha\mathbf{I}}\nonumber\\
&-2\langle x(\tau k)-\mathbf{x}^*,{\nabla} f(\tau k)-{\nabla} f(\mathbf{x}^*) \rangle.\label{second_appendix2}
\end{align}
Taking (\ref{second_appendix2}) into (\ref{second_appendix1}), we have
\begin{align}
&\Vert x(\tau k+\tau)-\mathbf{x}^*\Vert^2_{(\tau\alpha\mathbf{I})^{-1}}+\Vert d(k\tau+\tau)-\mathbf{d}^*\Vert^2_{M}\nonumber\\
\leq&\Vert x(\tau k)-\mathbf{x}^*\Vert^2_{(\tau\alpha\mathbf{I})^{-1}}-\Vert x(\tau k)-x(\tau k+\tau)\Vert^2_{(\tau\alpha\mathbf{I})^{-1}}\nonumber\\
&+\Vert d(\tau k)-\mathbf{d}^*\Vert^2_{M}-\Vert d(\tau k)-d(\tau k+\tau)\Vert^2_{M}\nonumber\\
&+{\Vert}  x(\tau k)-x(\tau k+\tau){\Vert}^2_{(\tau\alpha\mathbf{I})^{-1}}+{\Vert}  {\nabla} f(\mathbf{x}^*)-{\nabla} f(\tau k){\Vert}^2_{\tau\alpha\mathbf{I}}\nonumber\\
&-2\langle x(\tau k)-\mathbf{x}^*,{\nabla} f(\tau k)-{\nabla} f(\mathbf{x}^*) \rangle\nonumber\\
&+2\langle x(\tau k+\tau)-\mathbf{x}^*,A(k)+B(k) \rangle_{(\tau\alpha\mathbf{I})^{-1}}\nonumber\\
&-\Vert \tau\alpha d(\tau k+\tau)-\tau \alpha \mathbf{d}^*-A(k)-B(k){\Vert}^2_{(\tau\alpha\mathbf{I})^{-1}}.\label{previous_theorem1}
\end{align}
As for (\ref{previous_theorem1}), we need to analyze ${\Vert} {\nabla} f(\mathbf{x}^*)-{\nabla} f(\tau k){\Vert}^2_{\tau \alpha\mathbf{I}}-2\langle x(\tau k)-\mathbf{x}^*,{\nabla} f(\tau k)-{\nabla} f(\mathbf{x}^*) \rangle$ and $2\langle x(\tau k+\tau)-\mathbf{x}^*,A(k)+B(k) \rangle_{(\tau\alpha\mathbf{I})^{-1}}-\Vert \tau\alpha d(\tau k+\tau)-\tau \alpha \mathbf{d}^*-A(k)-B(k){\Vert}^2_{(\tau\alpha\mathbf{I})^{-1}}$. Firstly, from Assumption \ref{smooth_assumption} and Assumption \ref{strong_convex_assumption}, we know that
\begin{align}
&{\Vert} {\nabla} f(\mathbf{x}^*)-{\nabla} f(\tau k){\Vert}^2_{\tau \alpha\mathbf{I}}\nonumber\\
&-2\langle x(\tau k)-\mathbf{x}^*,{\nabla} f(\tau k)-{\nabla} f(\mathbf{x}^*) \rangle\nonumber\\
\leq &{\Vert} {\nabla} f(\mathbf{x}^*)-{\nabla} f(\tau k){\Vert}^2_{(\tau \alpha\mathbf{I})}-(2-\tau\alpha L)\Vert  x(\tau k)-\mathbf{x}^*\Vert^2_{\mu \mathbf{I}}\nonumber\\
&-(\tau\alpha L){\Vert} {\nabla} f(\mathbf{x}^*)-{\nabla} f(\tau k){\Vert}^2_{(L\mathbf{I})^{-1}}\nonumber\\
\leq&-(2-\tau\alpha L)\Vert  x(\tau k)-\mathbf{x}^*\Vert^2_{\mu \mathbf{I}}.\label{corollary_proof_1}
\end{align}
Then, we will analyze $2\langle x(\tau k+\tau)-\mathbf{x}^*,A(k)+B(k) \rangle_{(\tau\alpha\mathbf{I})^{-1}}-\Vert \tau\alpha d(\tau k+\tau)-\tau \alpha \mathbf{d}^*-A(k)-B(k){\Vert}^2_{(\tau\alpha\mathbf{I})^{-1}}$. We have 
\begin{align}
&2\langle x(\tau k+\tau)-\mathbf{x}^*,A(k)+B(k) \rangle_{(\tau\alpha\mathbf{I})^{-1}}\nonumber\\
&-\Vert \tau\alpha d(\tau k+\tau)-\tau \alpha \mathbf{d}^*-A(k)-B(k){\Vert}^2_{(\tau\alpha\mathbf{I})^{-1}}\nonumber\\
=&2{\langle} x(\tau k+\tau)-\mathbf{x}^*,A(k)+B(k) {\rangle}_{(\tau\alpha\mathbf{I})^{-1}}\nonumber\\
&+2{\langle} \tau\alpha d(\tau k+\tau)-\tau \alpha \mathbf{d}^*,A(k)+B(k){\rangle}_{(\tau\alpha\mathbf{I})^{-1}}\nonumber\\
&-\Vert  d(\tau k+\tau)- \mathbf{d}^*{\Vert}^2_{\tau\alpha\mathbf{I}}-\Vert A(k)+B(k){\Vert}^2_{(\tau\alpha\mathbf{I})^{-1}}\nonumber\\
\leq&2{\langle} x(\tau k+\tau)-\mathbf{x}^*,A(k)+B(k) {\rangle}_{(\tau\alpha\mathbf{I})^{-1}}\nonumber\\
&+2{\langle} \tau\alpha d(\tau k+\tau)-\tau\alpha \mathbf{d}^*,A(k)+B(k){\rangle}_{(\tau\alpha\mathbf{I})^{-1}}\nonumber\\
&-\Vert  d(\tau k+\tau)-\mathbf{d}^*{\Vert}^2_{\tau\alpha\mathbf{I}}-\Vert A(k)+B(k){\Vert}^2_{(\tau\alpha\mathbf{I})^{-1}}\nonumber\\
\leq&c_1{\Vert} x(\tau k+\tau)-\mathbf{x}^*{\Vert}^2_{(\tau\alpha\mathbf{I})^{-1}}+\frac{1}{c_1}{\Vert} A(k)+B(k){\Vert}^2_{(\tau\alpha\mathbf{I})^{-1}}\nonumber\\
&+c_2{\Vert} d(\tau k+\tau)-\mathbf{d}^*{\Vert}^2_{\tau \alpha\mathbf{I}}+\frac{1}{c_2}{\Vert} A(k)+B(k){\Vert}^2_{(\tau\alpha\mathbf{I})^{-1}}\nonumber\\
&-\Vert d(\tau k+\tau)- \mathbf{d}^*{\Vert}^2_{\tau\alpha\mathbf{I}}-\Vert A(k)+B(k){\Vert}^2_{(\tau\alpha\mathbf{I})^{-1}}\nonumber\\
\leq& (\frac{1}{c_1}+\frac{1}{c_2}-1 )\Vert A(k)+B(k){\Vert}^2_{(\tau\alpha\mathbf{I})^{-1}}\nonumber\\
&+c_1{\Vert} x(\tau k+\tau)-\mathbf{x}^*{\Vert}^2_{(\tau\alpha\mathbf{I})^{-1}}\nonumber\\
&+(c_2-1){\Vert} d(\tau k+\tau)-\mathbf{d}^*{\Vert}^2_{\tau \alpha\mathbf{I}}.\label{corollary_proof_2}
\end{align}
As for the term $\Vert A(k)+B(k){\Vert}^2_{(\tau\alpha\mathbf{I})^{-1}}$, we have the following lemma.
\begin{lemma}\label{most_important2}
Thus, we know that
\begin{align*}
&\Vert A(k)+B(k){\Vert}^2_{(\tau\alpha\mathbf{I})^{-1}}\\
\leq&{2\alpha\tau}\Vert d(\tau k+\tau)-d(\tau k){\Vert}^2+B_1 L^4\alpha^3\Vert x(\tau k)- x^*\Vert^2\\
&+B_1L^2\alpha^3\Vert d(\tau k)- \mathbf{d}^*\Vert^2,
\end{align*}
where $B_1={\tau^3}(1+\frac{2}{\tau})^{2\tau-2}$.
\end{lemma}
\begin{proof}
See Appendix \ref{most_important2_proof}.
\end{proof}
From (\ref{corollary_proof_1}), (\ref{corollary_proof_2}), and Lemma \ref{most_important2}, we know that
\begin{align}
&\Vert x(\tau k+\tau)-\mathbf{x}^*\Vert^2_{(\tau\alpha\mathbf{I})^{-1}}+\Vert d(k\tau+\tau)-\mathbf{d}^*\Vert^2_{M}\nonumber\\
\leq&\Vert x(\tau k)-\mathbf{x}^*\Vert^2_{(\tau\alpha\mathbf{I})^{-1}}+\Vert d(\tau k)-\mathbf{d}^*\Vert^2_{M}\nonumber\\
&-\Vert d(\tau k)-d(\tau k+\tau)\Vert^2_{M}\nonumber-(2-\tau\alpha L)\Vert  x(\tau k)-\mathbf{x}^*\Vert^2_{\mu \mathbf{I}}\nonumber\\
&+c_1{\Vert} x(\tau k+\tau)-\mathbf{x}^*{\Vert}^2_{(\tau\alpha\mathbf{I})^{-1}}\nonumber\\
&+(c_2-1){\Vert} d(\tau k+\tau)-\mathbf{d}^*{\Vert}^2_{\tau \alpha\mathbf{I}}\nonumber\\
&+(\frac{1}{c_1}+\frac{1}{c_2}-1)\Big\{{2\alpha\tau}\Vert d(\tau k+\tau)-d(\tau k){\Vert}^2\nonumber\\
&+B_1  L^4\alpha^3\Vert x(\tau k)- x^*\Vert^2+B_1 L^2\alpha^3\Vert d(\tau k)- \mathbf{d}^*\Vert^2\Big\}.\label{appendix_1_e4}
\end{align}
Rearranging terms in (\ref{appendix_1_e4}), we have
\begin{align*}
&{\Vert} x(\tau k+\tau)-\mathbf{x}^*{\Vert}^2_{(\tau\alpha\mathbf{I})^{-1}}+{\Vert} d(k\tau+\tau)-\mathbf{d}^*{\Vert}^2_{M}\nonumber\\
\leq&{\Vert} x(\tau k)-\mathbf{x}^*{\Vert}^2_{(\tau\alpha\mathbf{I})^{-1}}+{\Vert} d(\tau k)-\mathbf{d}^*{\Vert}^2_{M}\nonumber\\
&-{\Vert} d(\tau k)-d(\tau k+\tau){\Vert}^2_{M}\nonumber\\
&-(2-\tau\alpha L)\tau\mu\alpha{\Vert}  x(\tau k)-\mathbf{x}^*{\Vert}^2_{(\tau\alpha\mathbf{I})^{-1}}\nonumber\\
&+c_1{\Vert} x(\tau k+\tau)-\mathbf{x}^*{\Vert}^2_{(\tau\alpha\mathbf{I})^{-1}}\nonumber\\
&+(c_2-1){\Vert} d(\tau k+\tau)-\mathbf{d}^*{\Vert}^2_{\tau\alpha\mathbf{I}}\nonumber\\
&+(\frac{1}{c_1}+\frac{1}{c_2}-1){2\alpha\tau}{\Vert} d(\tau k+\tau)-d(\tau k){\Vert}^2\\
&+(\frac{1}{c_1}+\frac{1}{c_2}-1)B_1\tau \alpha^4 L^4{\Vert} x(\tau k)- x^*{\Vert}^2_{(\tau\alpha\mathbf{I})^{-1}}\\
&+(\frac{1}{c_1}+\frac{1}{c_2}-1)\frac{B_1}{\tau} L^2\alpha^2{\Vert} d(\tau k)- \mathbf{d}^*{\Vert}^2_{\tau\alpha\mathbf{I}}.
\end{align*}
Equivalently, we know that
\begin{align*}
&(1-c_1){\Vert} x(\tau k+\tau)-\mathbf{x}^*{\Vert}^2_{(\tau\alpha\mathbf{I})^{-1}}\\
&+{\Vert} d(k\tau+\tau)-\mathbf{d}^*{\Vert}^2_{M+(1-c_2)\tau\alpha \mathbf{I}}\\
\leq&{\{}1-(2-\tau\alpha L)\tau\mu\alpha\\
&+(\frac{1}{c_1}+\frac{1}{c_2}-1)B_1\tau \alpha^4 L^4 {\}}{\Vert} x(\tau k)-\mathbf{x}^*{\Vert}^2_{(\tau\alpha\mathbf{I})^{-1}}\\
&+{\Vert} d(\tau k)-\mathbf{d}^*{\Vert}^2_{M+(\frac{1}{c_1}+\frac{1}{c_2}-1){B_1} L^2\alpha^3 \mathbf{I}}\\
&-{\Vert} d(\tau k)-d(\tau k+\tau){\Vert}^2_{M}\nonumber\\
&+(\frac{1}{c_1}+\frac{1}{c_2}-1){2\alpha\tau}{\Vert} d(\tau k+\tau)-d(\tau k){\Vert}^2.
\end{align*}
We select $c_1=c_2=\tau\mu\alpha$ and then we know that
\begin{align}
&(1-\tau\mu\alpha){\Vert} x(\tau k+\tau)-\mathbf{x}^*{\Vert}^2_{(\tau\alpha\mathbf{I})^{-1}}\nonumber\\
&+{\Vert} d(k\tau+\tau)-\mathbf{d}^*{\Vert}^2_{M+(1-\tau\mu\alpha)\tau\alpha\mathbf{I}}\nonumber\\
\leq&{\Big\{}1-2\tau\mu\alpha+ L\mu\tau^2\alpha^2+\frac{2\tau^3}{\mu}(1+\frac{2}{\tau})^{2\tau-2} L^4\alpha^3\nonumber\\
&-\tau^4(1+\frac{2}{\tau})^{2\tau-2} L^4\alpha^4 {\Big\}}{\Vert} x(\tau k)-\mathbf{x}^*{\Vert}^2_{(\tau\alpha\mathbf{I})^{-1}}\nonumber\\
&+{\Vert} d(\tau k)-\mathbf{d}^*{\Vert}^2_{M+(\frac{2}{\tau\mu\alpha}-1){\tau^3(1+\frac{2}{\tau})^{2\tau-2}} L^2\alpha^3\mathbf{I}}\nonumber\\
&-{\Vert} d(\tau k)-d(\tau k+\tau){\Vert}^2_{M}\nonumber\\
&+(\frac{2}{\tau\mu\alpha}-1){2\alpha\tau}{\Vert} d(\tau k+\tau)-d(\tau k){\Vert}^2.\label{final_theorem1}
\end{align}

We need the following lemma from \cite{Ming_Yan1} to design appropriate $c$ satisfying  
\begin{align*}
&-{\Vert} d(\tau k)-d(\tau k+\tau){\Vert}^2_{M}\nonumber\\
&+(\frac{2}{\tau\mu\alpha}-1){2\alpha\tau}{\Vert} d(\tau k+\tau)-d(\tau k){\Vert}^2\leq 0.
\end{align*}

\begin{lemma}
Let $P=c^{-1}(\mathbf{I}-\frac{1}{N}\mathbf{1}_{N}\mathbf{1}^{\bf T}_{N})^{\dagger}-b\mathbf{I}$ with $\mathbf{I}\succeq cb(\mathbf{I}-\frac{1}{N}\mathbf{1}_{N}\mathbf{1}^{\bf T}_{N})\succeq \mathbf{0}_{N\times N}$. Then, $\Vert\cdot\Vert_{P}$ is a norm defined for \textbf{range}$(\mathbf{I}-\frac{1}{N}\mathbf{1}_{N}\mathbf{1}^{\bf T}_{N})$.
\end{lemma}
We know that
 \begin{align*}
&-{\Vert} d(\tau k)-d(\tau k+\tau){\Vert}^2_{M}\\
&+(\frac{2}{\tau\mu\alpha}-1){2\alpha\tau}{\Vert} d(\tau k+\tau)-d(\tau k){\Vert}^2\\
=&-{\Vert} d(\tau k)-d(\tau k+\tau){\Vert}^2_{M-(\frac{2}{\tau\mu\alpha}-1){2\alpha\tau}\mathbf{I}}\\
=&-{\Vert} d(\tau k)-d(\tau k+\tau){\Vert}^2_{M-(\frac{4}{\mu}-2\alpha\tau) \mathbf{I}}
\end{align*}
Then, we know that
\begin{align*}
&M+2\alpha\tau \mathbf{I}-\frac{4}{\mu}\mathbf{I}\\
=&c^{-1}(\mathbf{I}-\frac{1}{N}\mathbf{1}_{N}\mathbf{1}^{\bf T}_{N})^{\dagger}-\alpha I+2\alpha\tau I-\frac{4}{\mu}\mathbf{I}\\
=&c^{-1}(\mathbf{I}-\frac{1}{N}\mathbf{1}_{N}\mathbf{1}^{\bf T}_{N})^{\dagger}-(\alpha+\frac{4}{\mu}-2\alpha\tau)\mathbf{I}.
\end{align*}
For $x\in$\textbf{range}$(I-\frac{1}{N}\mathbf{1}_{N}\mathbf{1}^{\bf T}_{N})$, $\Vert x\Vert_{M-(\frac{4}{\mu}-2\alpha\tau) \mathbf{I}}$ is the norm if 
\begin{align*}
\mathbf{I}-c\Big(\alpha+\frac{4}{\mu}-2\alpha\tau\Big)(\mathbf{I}-\frac{1}{N}\mathbf{1}_{N}\mathbf{1}^{\bf T}_{N})
\end{align*}
is semipositive. The sufficient condition to ensure the semipositive property of $\mathbf{I}-c(\alpha+\frac{4}{\mu}-2\alpha\tau)(\mathbf{I}-\frac{1}{N}\mathbf{1}_{N}\mathbf{1}^{\bf T}_{N})$ is
\begin{align*}
1-2c\Big(\alpha+\frac{4}{\mu}-2\alpha\tau\Big)\geq 0.
\end{align*}
Thus, we know that
\begin{align*}
0<c\leq\frac{\mu}{2\mu\alpha+8}.
\end{align*}
is a sufficient condition to ensure $\Vert x\Vert_{M-(\frac{4}{\mu}-2\alpha\tau) \mathbf{I}}$ being the norm for $x\in$\textbf{range}$(\mathbf{I}-\frac{1}{N}\mathbf{1}_{N}\mathbf{1}^{\bf T}_{N})$. From (\ref{iteration_matrix1}), we know that $d(\tau k)-d(\tau k+\tau)\in$\textbf{range}$(\mathbf{I}-\frac{1}{N}\mathbf{1}_{N}\mathbf{1}^{\bf T}_{N})$. Thus, we know that
\begin{align}
&-{\Vert} d(\tau k)-d(\tau k+\tau){\Vert}^2_{M}\nonumber\\
&+(\frac{2}{\tau\mu\alpha}-1){2\alpha\tau}{\Vert} d(\tau k+\tau)-d(\tau k){\Vert}^2\leq 0.\label{final_inequality}
\end{align}
if $0<c\leq\frac{\mu}{2\mu\alpha+8}$. Thus, from (\ref{final_theorem1}) and (\ref{final_inequality}), we know that
\begin{align*}
&(1-\tau\mu\alpha){\Vert} x(\tau k+\tau)-\mathbf{x}^*{\Vert}^2_{(\tau\alpha\mathbf{I})^{-1}}\\
&+{\Vert} d(k\tau+\tau)-\mathbf{d}^*{\Vert}^2_{M+(1-\tau\mu\alpha)\tau\alpha\mathbf{I}}\\
\leq&{\Big\{}1-2\tau\mu\alpha+ L\mu\tau^2\alpha^2+\frac{2\tau^3}{\mu}(1+\frac{2}{\tau})^{2\tau-2} L^4\alpha^3\\
&-\tau^4(1+\frac{2}{\tau})^{2\tau-2} L^4\alpha^4 {\Big\}}{\Vert} x(\tau k)-\mathbf{x}^*{\Vert}^2_{(\tau\alpha\mathbf{I})^{-1}}\\
&+{\Vert} d(\tau k)-\mathbf{d}^*{\Vert}^2_{M+(\frac{2}{\tau\mu\alpha}-1){\tau^3(1+\frac{2}{\tau})^{2\tau-2}} L^2\alpha^3\mathbf{I}}.
\end{align*}
The proof of Theorem \ref{most_important} is complete.

\section{Proof of Corollary \ref{corollary1}}\label{proof_corollary1}

From Theorem \ref{most_important}, we need to find the solution of $\alpha$ satisfying 
\begin{equation}\label{inequality3}
\left\{
\begin{aligned}
1-\tau\mu\alpha>&1-(2-\tau\alpha L)\tau\mu\alpha+\Big(\frac{2}{\tau\mu\alpha}-1\Big){\tau^2}B_2 \alpha^4 L^4\\
1-\tau\mu\alpha>&(\frac{2}{\tau\mu\alpha}-1)B_2 L^2\alpha^2,
\end{aligned}
\right.
\end{equation}
where $B_2={\tau^2}(1+\frac{2}{\tau})^{2\tau-2}$. Algorithm \ref{step_size_search} is designed to search the appropriate learning rate satisfying (\ref{inequality3}). To prove this point, we need to show that
\begin{itemize}
    \item [\textbf{(i)}] $\alpha<\min\{\frac{1}{2\tau L},\frac{\mu^2}{2\tau(1+\frac{2}{\tau})^{2\tau-2}L^3},\frac{\mu}{5\tau(1+\frac{2}{\tau})^{2\tau-2} L^2}\}$ is a solution of (\ref{inequality3}).
    \item [\textbf{(ii)}] Algorithm \ref{step_size_search} will be complete in finite steps, i.e., there exists $\alpha>\min\{\frac{1}{2\tau L},\frac{\mu^2}{2\tau(1+\frac{2}{\tau})^{2\tau-2}L^3},\frac{\mu}{5\tau(1+\frac{2}{\tau})^{2\tau-2} L^2}\}$ not satisfying (\ref{inequality3}).
\end{itemize}
We will prove \textbf{(i)} firstly. 
\begin{itemize}
\item As for the first inequality of (\ref{inequality3}), we have
\begin{align*}
\tau\mu\alpha<(2-\tau\alpha L)\tau\mu\alpha-\Big(\frac{2}{\tau\mu\alpha}-1\Big)B_2\tau^2\alpha^4L^4.
\end{align*}
The sufficient condition is
\begin{align*}
\tau\mu\alpha>\tau^2\mu L\alpha^2+\Big(\frac{2L^4{\tau^3}}{\mu}\Big)(1+\frac{2}{\tau})^{2\tau-2}\alpha^3.
\end{align*}
since $B_2={\tau^2}(1+\frac{2}{\tau})^{2\tau-2}$. If the learning rate $\alpha$ satisfies
\begin{align*}
\alpha<\frac{\mu^2}{2L^3\tau(1+\frac{2}{\tau})^{2\tau-2}},
\end{align*}
the sufficient condition of is $\tau\mu\alpha>2\tau^2L\mu\alpha^2$. Thus, if $\alpha$ satisfies 
\begin{align}\label{step_size_condition1}
\alpha<\min\Big\{\frac{1}{2\tau L},\frac{\mu^2}{2L^3\tau(1+\frac{2}{\tau})^{2\tau-2}}\Big\},
\end{align}
the first inequality of (\ref{inequality3}) is satisfied.
\item  As for the second  inequality of (\ref{inequality3}), we have
\begin{align*}
1-\tau\mu\alpha>(\frac{2}{\tau\mu\alpha}-1){\tau^2}(1+\frac{2}{\tau})^{2\tau-2} L^2 \alpha^2.
\end{align*}
Equivalently, we have
\begin{align*}
(\frac{2}{\tau\mu\alpha}-1){\tau^2}(1+\frac{2}{\tau})^{2\tau-2} L^2 \alpha^2+\tau\mu\alpha-2<-1.
\end{align*}
Equivalently, we have
\begin{align*}
(\frac{2-\tau\mu\alpha}{\mu}){\tau}(1+\frac{2}{\tau})^{2\tau-2} L^2 \alpha-(2-\tau\mu\alpha)<-1.
\end{align*}
Equivalently, we have
\begin{align*}
(2-\tau\mu\alpha)\Big\{\frac{\tau}{\mu}(1+\frac{2}{\tau})^{2\tau-2} L^2 \alpha-1\Big\}<-1.
\end{align*}
Equivalently, we have
\begin{align*}
(2-\tau\mu\alpha)\Big\{1-(\frac{\tau}{\mu})(1+\frac{2}{\tau})^{2\tau-2} L^2 \alpha\Big\}>1.
\end{align*}
Equivalently, we can find the sufficient condition
\begin{align*}
1-(\frac{\tau}{\mu})(1+\frac{2}{\tau})^{2\tau-2} L^2 \alpha&>0.8,\\
2-\tau\mu\alpha&>1.5.
\end{align*}
Thus, the sufficient condition for the second inequality of (\ref{inequality3}) is
\begin{align}\label{step_size_condition2}
\alpha<\min\Big\{\frac{1}{2\tau\mu},\frac{\mu}{5\tau(1+\frac{2}{\tau})^{2\tau-2} L^2}\Big\}
\end{align}

From (\ref{step_size_condition1}), (\ref{step_size_condition2}), and $\mu\leq L$, we know that the sufficient condition to the learning rate $\alpha$ to satisfy (\ref{inequality3}) is
\begin{align*}
\alpha<\min\{\frac{1}{2\tau L},\frac{\mu^2}{2\tau(1+\frac{2}{\tau})^{2\tau-2}L^3},\frac{\mu}{5\tau(1+\frac{2}{\tau})^{2\tau-2} L^2}\}.
\end{align*}
\end{itemize}
Thus, \textbf{(i)} is proved. Then, we will prove \textbf{(ii)}. If $\alpha=\frac{2}{\tau L}$, we know that the inequalities (\ref{inequality3}) is not satisfied. Thus, Algorithm \ref{step_size_search} will be complete in finite steps. Thus, combining with Theorem \ref{most_important}, there exist $0<\rho_1<1$ and $0<\rho_2<1$ such that
\begin{align*}
&(1-\tau\mu\alpha){\Vert} x(\tau k+\tau)-\mathbf{x}^*{\Vert}^2_{(\tau\alpha\mathbf{I})^{-1}}+{\Vert} d(k\tau+\tau)-\mathbf{d}^*{\Vert}^2_{M_1}\\
\leq&\rho_1(1-\tau\mu\alpha){\Vert} x(\tau k)-\mathbf{x}^*{\Vert}^2_{(\tau\alpha\mathbf{I})^{-1}}+\rho_2{\Vert} d(\tau k)-\mathbf{d}^*{\Vert}^2_{M_1}.
\end{align*}
where $M_1=M+(1-\tau\mu\alpha)\tau\alpha\mathbf{I}$,  
\begin{align*}
\rho_1=&\frac{1-(2-\tau\alpha L)\tau\mu\alpha+(\frac{2}{\tau\mu\alpha}-1)B_2\tau^2 \alpha^4 L^4 }{1-\tau\mu\alpha},\\
\rho_2=&\frac{\lambda_{max}\{M\}+(\frac{2}{\tau\mu\alpha}-1)B_2\alpha^2 L^2\tau\alpha}{\lambda_{max}\{M\}+(1-\tau\mu\alpha)\tau\alpha},
\end{align*}
and $\lambda_{max}\{M\}$ is the largest eigenvalue of $M$. We define $\rho=\max\{\rho_1,\rho_2\}$. Thus, we know that
\begin{align*}
&(1-\tau\mu\alpha){\Vert} x(\tau k)-\mathbf{x}^*{\Vert}^2_{(\tau\alpha\mathbf{I})^{-1}}+{\Vert} d(k\tau)-\mathbf{d}^*{\Vert}^2_{M_1}\\
\leq&\rho^{k}\Big\{(1-\tau\mu\alpha){\Vert} x(0)-\mathbf{x}^*{\Vert}^2_{(\tau\alpha\mathbf{I})^{-1}}+{\Vert} d(0)-\mathbf{d}^*{\Vert}^2_{M_1}\Big\}.
\end{align*}
Thus, we know that
\begin{align*}
{\Vert} x(\tau k)-\mathbf{x}^*{\Vert}^2\leq&\rho^k{\Vert} x(0)-\mathbf{x}^*{\Vert}^2+\frac{\rho^k\tau\alpha}{1-\tau\mu\alpha}{\Vert} d(0)-\mathbf{d}^*{\Vert}^2_{M_1}.
\end{align*}
The proof of Corollary \ref{corollary1} is complete.

\section{Proof of Lemma \ref{most_important2}}\label{most_important2_proof}
We have
\begin{align*}
&\Vert A(k)+B(k){\Vert}^2_{(\tau\alpha\mathbf{I})^{-1}}\\
=&\frac{1}{\tau\alpha}\Vert A(k)+B(k){\Vert}^2\\
\leq&\frac{2}{\tau\alpha}\Vert A(k){\Vert}^2+\frac{2}{\tau\alpha}\Vert B(k){\Vert}^2\\
\leq&\frac{2}{\tau\alpha}\Vert A(k){\Vert}^2+\frac{2\alpha(\tau-1)^2}{\tau}\Vert d(\tau k+\tau)-d(\tau k){\Vert}^2\\
\leq&\frac{2\alpha}{\tau}\Vert (\tau-1){\nabla} f(\tau k)-{\nabla} f(\tau k+1)-{\nabla} f(\tau k+2)\\
&-\cdots-{\nabla} f(\tau k+\tau-1){\Vert}^2\\
&+\frac{2\alpha(\tau-1)^2}{\tau}\Vert d(\tau k+\tau)-d(\tau k){\Vert}^2\\
\leq&\frac{2\alpha(\tau-1)^2}{\tau}\Vert d(\tau k+\tau)-d(\tau k){\Vert}^2\\
&+\frac{2(\tau-1)\alpha}{\tau}\{\Vert {\nabla} f(\tau k+1)-{\nabla} f(\tau k)\Vert^2\\
&+\Vert {\nabla} f(\tau k+2)-{\nabla} f(\tau k)\Vert^2\\
&+\cdots +\Vert {\nabla} f(\tau k+\tau-1)-{\nabla} f(\tau k)\Vert^2\}
\end{align*}
Then, we know that
\begin{align*}
&\Vert {\nabla} f(\tau k+1)-{\nabla} f(\tau k)\Vert\\
\leq&L\Vert  x(\tau k+1)- x(\tau k)\Vert\\
=&\alpha L\Vert {\nabla} f(\tau k)+ d(\tau k)\Vert\\
\leq&\alpha L\Vert {\nabla} f(\tau k)-{\nabla} f(\mathbf{x}^*)\Vert+\alpha L\Vert  d(\tau k)-\mathbf{d}^*\Vert\\
\leq&\alpha L^2\Vert x(\tau k)-\mathbf{x}^*\Vert+\alpha L\Vert  d(\tau k)-\mathbf{d}^*\Vert.
\end{align*}
Moreover, we have
\begin{align*}
&\Vert {\nabla} f(\tau k+2)-{\nabla} f(\tau k)\Vert\\
\leq&L\Vert  x(\tau k+2)- x(\tau k)\Vert\\
\leq& L\Vert x(\tau k+1)-x(\tau k)\Vert+L\Vert \alpha {\nabla} f(\tau k+1)+\alpha d(\tau k)\Vert\\
\leq& L\Vert x(\tau k+1)-x(\tau k)\Vert\\
&+L\Vert \alpha {\nabla} f(\tau k+1)-\alpha {\nabla} f(\mathbf{x}^*)\Vert+L\Vert\alpha d(\tau k)-\alpha \mathbf{d}^*\Vert\\
\leq& \alpha L^2\Vert x(\tau k)-\mathbf{x}^*\Vert+\alpha L\Vert  d(\tau k)-\mathbf{d}^*\Vert\\
&+\alpha L^2\Vert x(\tau k+1)- x^*\Vert+\alpha L\Vert  d(\tau k)- \mathbf{d}^*\Vert.
\end{align*}
From mathematical induction, we can derive that
\begin{align}
&\Vert {\nabla} f(\tau k+\tau -1)-{\nabla} f(\tau k)\Vert\leq\alpha L^2\{\Vert x(\tau k)-\mathbf{x}^*\Vert\nonumber\\
&+\Vert x(\tau k+1)- \mathbf{x}^*\Vert+\cdots+\Vert x(\tau k+\tau -2)- \mathbf{x}^*\Vert\}\nonumber\\
&+(\tau-1)\alpha L\Vert  d(\tau k)- \mathbf{d}^*\Vert.\label{MI_1}
\end{align}
Then, we consider
\begin{align*}
&\Vert x(\tau k+1)- \mathbf{x}^*\Vert\\
=&\Vert x(\tau k)-\alpha {\nabla} f(\tau k)-\alpha d(\tau k)- \mathbf{x}^*\Vert\\
\leq& \Vert x(\tau k)- \mathbf{x}^*\Vert+\Vert \alpha {\nabla} f(\tau k)+\alpha d(\tau k)\Vert\\
\leq& \Vert x(\tau k)- \mathbf{x}^*\Vert+\Vert \alpha {\nabla} f(\tau k)- \alpha{\nabla} f(\mathbf{x}^*)\Vert\\
&+\Vert  \alpha d(\tau k)-\alpha \mathbf{d}^*\Vert\\
\leq& (1+\alpha L)\Vert x(\tau k)- \mathbf{x}^*\Vert+\alpha\Vert d(\tau k)- \mathbf{d}^*\Vert.
\end{align*}
Moreover, we have
\begin{align*}
&\Vert x(\tau k+2)- x^*\Vert\\
=&\Vert x(\tau k+1)-\alpha {\nabla} f(\tau k+1)-\alpha d(\tau k)- x^*\Vert\\
\leq& \Vert x(\tau k+1)- x^*\Vert+\Vert \alpha {\nabla} f(\tau k+1)+\alpha d(\tau k)\Vert\\
\leq& \Vert x(\tau k+1)- x^*\Vert+\Vert \alpha {\nabla} f(\tau k+1)- \alpha {\nabla} f(\mathbf{x}^*)\Vert\\
&+\Vert  \alpha d(\tau k)-\alpha \mathbf{d}^*\Vert\\
\leq& (1+\alpha L)\Vert x(\tau k+1)- x^*\Vert+\alpha\Vert d(\tau k)- \mathbf{d}^*\Vert.
\end{align*}
From mathematical induction, we can derive that
\begin{align*}
&\Vert x(\tau k+\tau-2)- x^*\Vert\\
=&\Vert x(\tau k+\tau-3)-\alpha {\nabla} f(\tau k+\tau-3)-\alpha d(\tau k)- x^*\Vert\\
\leq& \Vert x(\tau k+\tau-3)- x^*\Vert+\Vert \alpha {\nabla} f(\tau k+\tau-3)- \alpha{\nabla} f(\mathbf{x}^*)\Vert\\
&+\Vert  \alpha d(\tau k)-\alpha \mathbf{d}^*\Vert\\
\leq& (1+\alpha L)\Vert x(\tau k+\tau-3)- x^*\Vert+\alpha\Vert d(\tau k)- \mathbf{d}^*\Vert.
\end{align*}
We know that $\alpha L<\frac{2}{\tau}$. Thus, we know that
\begin{align*}
\Vert x(\tau k+1)- x^*\Vert\leq &(1+\frac{2}{\tau})\Vert x(\tau k)- x^*\Vert\\
&+\alpha\Vert d(\tau k)- \mathbf{d}^*\Vert.
\end{align*}
We know that
\begin{align*}
&\Vert x(\tau k+2)- x^*\Vert\nonumber\\
\leq& (1+\frac{2}{\tau})\Vert x(\tau k+1)- x^*\Vert+\alpha\Vert d(\tau k)- \mathbf{d}^*\Vert\nonumber\\
\leq& (1+\frac{2}{\tau})\Big((1+\frac{2}{\tau})\Vert x(\tau k)- x^*\Vert+\alpha\Vert d(\tau k)- \mathbf{d}^*\Vert\Big)\nonumber\\
&+\alpha\Vert d(\tau k)- \mathbf{d}^*\Vert\nonumber\\
\leq& (1+\frac{2}{\tau})^2\Vert x(\tau k)- x^*\Vert+\Big((1+\frac{2}{\tau})\alpha+\alpha\Big)\Vert d(\tau k)- \mathbf{d}^*\Vert.
\end{align*}
Similarly, from mathematical induction, we can derive that
\begin{align}
\Vert x(\tau k&+\tau-2)- x^*\Vert\leq (1+\frac{2}{\tau})^{\tau-2}\Vert x(\tau k)- x^*\Vert\nonumber\\
&+\Big\{\frac{\tau}{2}\Big(1+\frac{2}{\tau}\Big)^{\tau-2}-\frac{\tau}{2}\Big\}\alpha\Vert d(\tau k)- \mathbf{d}^*\Vert.\label{MI_2}
\end{align}
Thus, from (\ref{MI_1}) and (\ref{MI_2}), we know that
\begin{align*}
&\Vert {\nabla} f(\tau k+\tau -1)-{\nabla} f(\tau k)\Vert\\
\leq& (\tau-1)\alpha L\Vert  d(\tau k)-\mathbf{d}^*\Vert+\alpha L^2\Vert x(\tau k)-\mathbf{x}^*\Vert\\
&+\alpha L^2\Vert x(\tau k+1)- x^*\Vert+\alpha L^2\Vert x(\tau k+2)- x^*\Vert\\
&\cdots\cdots\\
&+\alpha L^2\Vert x(\tau k+\tau -2)- x^*\Vert\\
\leq& (\tau-1)\alpha L\Vert  d(\tau k)-\mathbf{d}^*\Vert\\
&+\alpha L^2\Big\{1+(1+\frac{2}{\tau})+(1+\frac{2}{\tau})^2+\cdots\\
&+(1+\frac{2}{\tau})^{\tau-2}\Big\}\Vert x(\tau k)- x^*\Vert\\
&+\alpha L\Big\{1+(1+\frac{2}{\tau})+(1+\frac{2}{\tau})^2+\cdots\\
&+(1+\frac{2}{\tau})^{\tau-2}-(\tau-1)\Big\}\Vert d(\tau k)- \mathbf{d}^*\Vert\\
\leq&\frac{\alpha L\tau}{2}(1+\frac{2}{\tau})^{\tau-1}\Big\{L\Vert x(\tau k)- x^*\Vert+\Vert d(\tau k)- \mathbf{d}^*\Vert\Big\}.
\end{align*}
Moreover, we have
\begin{align*}
&\Vert {\nabla} f(\tau k+\tau -1)-{\nabla} f(\tau k)\Vert^2\\
\leq& A_1 L^4 \alpha^2\Vert x(\tau k)- x^*\Vert^2+A_1L^2\alpha^2 \Vert d(\tau k)- \mathbf{d}^*\Vert^2,
\end{align*}
where $A_1=\frac{\tau^2}{2}(1+\frac{2}{\tau})^{2\tau-2}$. Then, we have
\begin{align*}
&\Vert A(k)+B(k){\Vert}^2_{(\tau\alpha\mathbf{I})^{-1}}\\
\leq&\frac{2\alpha(\tau-1)^2}{\tau}\Vert d(\tau k+\tau)-d(\tau k){\Vert}^2\\
&+\frac{2(\tau-1)\alpha}{\tau}\Big\{\Vert {\nabla} f(\tau k+1)-{\nabla} f(\tau k)\Vert^2\\
&+\Vert {\nabla} f(\tau k+2)-{\nabla} f(\tau k)\Vert^2\\
&+\cdots +\Vert {\nabla} f(\tau k+\tau-1)-{\nabla} f(\tau k)\Vert^2\Big\}\\
\leq&\frac{2\alpha(\tau-1)^2}{\tau}\Vert d(\tau k+\tau)-d(\tau k){\Vert}^2\\
&+\frac{2(\tau-1)^2\alpha}{\tau}\Big\{A_1 L^4 \alpha^2\Vert x(\tau k)- x^*\Vert^2\\
&+A_1L^2\alpha^2 \Vert d(\tau k)- \mathbf{d}^*\Vert^2 \Big\}\\
\leq&{2\alpha\tau}\Vert d(\tau k+\tau)-d(\tau k){\Vert}^2+2A_1\tau  L^4\alpha^3\Vert x(\tau k)- x^*\Vert^2\\
&+2\tau A_1L^2\alpha^3\Vert d(\tau k)- \mathbf{d}^*\Vert^2.
\end{align*}
The proof of Lemma \ref{most_important2} is complete.



%\bibliographystyle{IEEEtran}

\bibliographystyle{ieeetr}
\bibliography{reference}

\end{document}